\documentclass[final,5p,times,twocolumn,authoryear]{elsarticle}

\usepackage{amssymb}
\usepackage{amsthm}
\usepackage{amsmath}
\usepackage{enumerate}
\usepackage{booktabs}
\usepackage{hyperref}
\usepackage{multirow}
\newtheorem{theorem}{Theorem}
\newtheorem{corollary}{Corollary}
\newtheorem{lemma}{Lemma}
\newtheorem{proposition}{Proposition}
\newtheorem{definition}{Definition}
\newtheorem{property}{Property}

\newtheorem{assumption}{Assumption}

\newcommand{\R}{\mathbb{R}}

\DeclareMathOperator{\dive}{div}

\newcommand{\norm}[1]{\left\lVert#1\right\rVert}

\date{}

\begin{document}

\begin{frontmatter}



\title{Approximation capabilities of measure-preserving neural networks}


\author[cas,ucas]{Aiqing Zhu}
\author[cas,ucas]{Pengzhan Jin}
\author[cas,ucas]{Yifa Tang\corref{cor}}
\cortext[cor]{Corresponding author: tyf@lsec.cc.ac.cn (Yifa Tang)}

\address[cas]{LSEC, ICMSEC, Academy of Mathematics and Systems Science, Chinese Academy of Sciences, Beijing 100190, China}
\address[ucas]{School of Mathematical Sciences, University of Chinese Academy of Sciences, Beijing 100049, China}

\begin{abstract}
Measure-preserving neural networks are well-developed invertible models, however, their approximation capabilities remain unexplored. This paper rigorously analyses the approximation capabilities of existing measure-preserving neural networks including NICE and RevNets. It is shown that for compact $U \subset \R^D$ with $D\geq 2$, the measure-preserving neural networks are able to approximate arbitrary measure-preserving map $\psi: U\to \R^D$ which is bounded and injective in the $L^p$-norm. In particular, any continuously differentiable injective map with $\pm 1$ determinant of Jacobian are measure-preserving, thus can be approximated.
\end{abstract}



\begin{keyword}
measure-preserving \sep neural networks \sep dynamical systems \sep approximation theory


\end{keyword}

\end{frontmatter}



\section{Introduction}\label{sec:Induction}
Deep neural networks have become an increasingly successful tool in modern machine learning applications and yielded transformative advances across diverse scientific disciplines \citep{krizhevsky2017imagenet,lecun2015deep,lu2021deepxde,schmidhuber2015deep}. It is well known that fully connected neural networks can approximate continuous mappings \citep{cybenko1989approximation,hornik1990universal}. Nevertheless, more sophisticated structures are preferred in practice,
and often yield surprisingly good performance \citep{behrmann2019invertible,chen2018neural,dinh2015nice,fiori2011numerical,fiori2011theoretical,gomez2017the,jin2020learning,jin2020sympnets}, such as convolutional neural
networks (CNNs) for image classification \citep{krizhevsky2012imagenet}, recurrent neural networks (RNNs) for natural language processing \citep{maas2013rectifier}, as well as residual neural networks (ResNets) \citep{he2016deep}, which allow information to be passed directly through for making less exploding or vanishing.

Recently, invertible models have attached increasing attention.  As the abilities of tracking of changes in probability density, they have been applied in many tasks, including generative models and variational inference \citep{behrmann2019invertible,chen2018neural,chen2019residual,dinh2017density,Rezende2015variational,kingma2018glow}. The learning model for the above use cases need to be invertible and expressive, as well as efficient for computation of Jacobian determinants. Additionally, more invertible structures are proposed for specific tasks. For example, \cite{gomez2017the} propose reversible residual networks (RevNets) to avoid storing intermediate activations during backpropagation relied on the invertible architecture, \cite{jin2020sympnets} develop symplectic-preserving networks for indentifying Hamiltonian systems.

To maintain the invertibility, most aforementioned architectures have other intrinsic regularizations or constraints, such as orientation-preserving \citep{behrmann2019invertible,chen2018neural}, symplectic-preserving \citep{jin2020sympnets}, as well as measure-preserving \citep{dinh2015nice,gomez2017the}. Encoding such structured information makes the classical universal approximation theorem no longer applicable. Recently, there have been many research works focusing on representations of such structured neural networks and developing fruitful results. \cite{jin2020sympnets} prove that SympNets can approximate arbitrary symplectic
maps based on appropriate activation functions. \cite{zhang2020approximation} analyze the approximation capabilities of Neural ODEs \citep{chen2018neural} and invertible residual networks \citep{behrmann2019invertible}, and give negative results (also given in \citep{dupont2019augmented}). \cite{kong2020the} explore the representation of a class of normalizing flow and show the universal approximation
properties of plane flows \citep{Rezende2015variational} when dimension
$d = 1$.

Measure-preserving (also known as volume-preserving, area-preserving) neural networks are well-developed invertible models. Their inverse and Jacobian determinants can be computed efficiently, thus they have practical applications \citep{dinh2015nice, gomez2017the,jin2020learning,zhang2021ivpf}. Due to measure-preserving constraints, there have been many works dedicated to enhance performance via improving expressivity \citep{dinh2017density,chen2018neural,chen2019residual, huang2018neural,kingma2018glow}. However, to the best of our knowledge, the approximation capability of measure-preserving neural networks, i.e., whether they can approximate any invertible measure-preserving map, remains unexplored mathematically.

This paper provides a rigorous mathematical theory to answer the above question. The architecture we investigated is the composition of the following modules,
\begin{equation}\label{eq:up}
\begin{aligned}
&\hat{x}[\ :s] = x[\ :s] + f_{net_1}(x[s:\ ]),\\ &\hat{x}[s:\ ] = x[s:\ ],
\end{aligned}
\end{equation}
and
\begin{equation}\label{eq:low}
\begin{aligned}
&\hat{x}[\ :s ] = x[\ :s ],\\
&\hat{x}[s:\ ] = x[s:\ ] + f_{net_2}(x[\ :s ]),
\end{aligned}
\end{equation}
which are the basic modules of NICE \citep{dinh2015nice} and RevNets \citep{gomez2017the}.
The main contribution of this work is to prove the approximation capabilities of above modules. It is shown that for compact $U \subset \R^D$ with $D\geq 2$, the measure-preserving neural networks are able to approximate arbitrary measure-preserving map $\psi: U\to \R^D$ which is bounded and injective in the $L^p$-norm. Note that measure-preserving neural networks are also bounded and injective on compact set. Specifically, the approximation theory holds for continuously differentiable injective maps with $\pm 1$ determinants of Jacobians.

The rest of this paper is organized as follows. Some preliminaries, including notations, definitions and existing network architectures are detailed in Section \ref{sec:Preliminaries}. In Section \ref{sec:Main results}, we present the approximation results. In Section \ref{sec:Discussions}, we perform numerical experiments to demonstrate the validity of learning measure-preserving map and discuss the application scopes of our theory. In Section \ref{sec:Proofs}, we present detailed proofs. Finally, we conclude this paper in Section \ref{sec:Summary}.

\section{Preliminaries}\label{sec:Preliminaries}
\subsection{Notations and definitions}
For convenience we collect together some of the notations introduced throughout the paper.
\begin{itemize}

    \item Range indexing notations, the same kind for Pytorch tensors, are employed throughout this paper. Details are presented in Table \ref{tab:notations}.

    \item For differentiable $F=(F_1,\cdots,F_{D})^\top: \R^D \rightarrow \R^D$, we denote by $J_F$ the Jacobian of $F$, i.e.,
    \begin{equation*}
    J_F \in \R^{D \times D} \ \text{and}\ J_F[i][j]=\frac{\partial F_i}{\partial x_j}.
    \end{equation*}

    \item For $1\leq p <\infty$, $U \subset \R^D$, $L^p(U)$ denotes the space of $p$-integrable measureable functions $F=(F_1,\cdots,F_{D})^\top: U \rightarrow \R^D$  for which the norm
        \begin{equation*}
        \norm{F}_{L^p(U)}=\sum_{d=1}^D (\int_{U} \lvert F_d(x) \rvert^p dx )^{\frac{1}{p}}
        \end{equation*}
        is finite; $C(U)$ consists of all continuous functions $F=(F_1,\cdots,F_{D})^\top: U \rightarrow \R^D$ with norm
        \begin{equation*}
            \norm{F}_{U}=\max_{1 \leq d \leq D}\sup_{x\in U} \lvert F_d(x) \rvert
        \end{equation*}
        on compact $U$.
    \item We denote by $\overline{\Omega}_{L^p(U)}$ the closure of $\Omega$ in $L^p(U)$ if $\Omega \subset L^p(U)$, meanwhile, denote by $\overline{\Omega}_{U}$ the closure of $\Omega$ in $C(U)$ if $\Omega \subset C(U)$.

    \item A function $f$ on $U$ is called Lipschitz if $\norm{f(x)-f(x')} \leq L\norm{x-x'}$ holds for all $x,x' \in U$.

    \item $\mathcal{NN}^{d}$ consists of some neural networks $f_{net}:\R^d\to\R^{D-d}$, we call it control family.
\end{itemize}

\begin{definition}
Let $U \subset \R^D$ be a Borel set. The Borel map $ \psi: U \rightarrow \R^D$ is (Lebesgue) measure-preserving if $\psi(U)$ is a Borel set and $\mathcal{H}[\psi^{-1}(B)] = \mathcal{H}[B]$ for all Borel sets $B \subset \psi(U)$, where $\mathcal{H}$ is Lebesgue measure.
\end{definition}
By the transformation formula for integrals, $\psi$ is measure-preserving if $\psi$ is injective, continuously differentiable and $\det(J_{\psi})=1$. The Jacobians of both (\ref{eq:up}) and (\ref{eq:low}) obey determinant identity and the composition of measure-preserving maps is again measure-preserving; a continuous $f_{net}$ can be approximated by smooth functions, thus the measure is also preserved by (\ref{eq:up}) and (\ref{eq:low}) with nondifferentiable control family due to the
dominated convergence theorem. Therefore the aforementioned architectures are measure-preserving and we call such learning models as measure-preserving neural networks.
\begin{table}[htbp]
    \centering
    \begin{tabular}{p{65pt}|p{150pt}}
        \toprule
        $x[i]$ & The $i$-th component (row) of vector (matrix) $x$.\\
        \midrule
        $x[:][j]$ & The $j$-th column of matrix $x$.\\
        \midrule
        $x[i_1:i_2]$ & $(x[i_1],\cdots, x[i_2-1])^\top$ if $x$ is a column vector or $(x[i_1],\cdots, x[i_2-1])$ if $x$ is a row vector, i.e., components from $i_1$ inclusive to $i_2$ exclusive. \\
        \midrule
        $x[\ :i ]$ and $x[i:\ ]$& $x[1:i]$ and  $x[i:D+1]$ for $x \in \R^D$, respectively. \\
        \midrule
        $\overline{x[i_1:i_2]}$ & $(x[\ :i_1]^\top, x[i_2:\ ]^\top)^\top$ if $x$ is a column vector or $(x[\ :i_1], x[i_2:\ ])$ if $x$ is a row vector, i.e., components in the vector $x$ excluding $x[i_1:i_2]$.\\
        \bottomrule
    \end{tabular}
    \caption{Range indexing notations in this paper.}
    \label{tab:notations}
\end{table}

\subsection{Measure-preserving neural networks}
We first briefly present existing measure-preserving neural networks as follows, including NICE \citep{dinh2015nice} and RevNet \citep{gomez2017the}.

NICE is an architecture to unsupervised generative modeling via learning a nonlinear bijective transformation between the data space and a latent space.
The architecture is composed of a series of modules which take inputs $(x_1,x_2)$ and produce outputs $(\hat{x}_1,\hat{x}_2)$ according to the following additive coupling rules,
\begin{equation}\label{mod:nice}
\begin{aligned}
&\hat{x}_1 = x_1 + f_{net}(x_2),\\
&\hat{x}_2 = x_2.
\end{aligned}
\end{equation}
Here, $f_{net}$ is typically a neural network, $x_1$ and $x_2$ form a partition of the vector in each layer. Since the model is invertible and its Jacobian has unit determinant, the log-likelihood and its gradient can be tractably computed. As an alternative, the components of inputs can be reshuffled before separating them. Clearly, this architecture is imposed measure-preserving constraints.

A similar architecture is used in the reversible residual network (RevNet) \citep{gomez2017the} which is a variant of ResNets \citep{he2016deep} to avoid storing intermediate activation during
backpropagation relied on the invertible architecture. In each module, the inputs are decoupled into $(x_1,x_2)$ and the outputs $(\hat{x}_1,\hat{x}_2)$ are produced by
\begin{equation}\label{mod:revnet}
\begin{aligned}
&\hat{x}_1 = x_1 + f_{net_1}(x_2),\\
&\hat{x}_2 = x_2 + f_{net_2}(\hat{x}_1).\\
\end{aligned}
\end{equation}
Here, $f_{net_1},f_{net_2}$ are trainable neural networks. It is observed that (\ref{mod:revnet}) is composed of two modules defined in (\ref{mod:nice}) with the given reshuffling operation before the second module and also measure-preserving.

The architecture we investigate is analogous to RevNet but without reshuffling operations and using fixed dimension-splitting  mechanisms in each layer. Let us begin by introducing the modules sets. Given integers $D\geq s\geq 2$ and control families $\mathcal{NN}^{D-s+1}, \mathcal{NN}^{s-1}$, denote
\begin{equation*}
\begin{aligned}
&\mathcal{M}_{up} = \left\{m: x\mapsto \hat{x}\ \left|
\begin{aligned}
\ & \hat{x}[\ :s] = x[\ :s] + f_{net}(x[s:\ ]),\\
& \hat{x}[s:\ ] = x[s:\ ],\\
& f_{net}\in \mathcal{NN}^{D-s+1}.
\end{aligned}
\right. \right\},\\
&\mathcal{M}_{low} = \left\{m: x\mapsto \hat{x}\ \left|
\begin{aligned}
\ &\hat{x}[\ :s] = x[\ :s],\\
&\hat{x}[s:\ ] = x[s:\ ] + f_{net}(\hat{x}[\ :s]),\\
& f_{net}\in \mathcal{NN}^{s-1}
\end{aligned}
\right. \right\}.
\end{aligned}
\end{equation*}
Subsequently, we define the collection of measure-preserving neural networks generated by $\mathcal{M}_{up}$ and $\mathcal{M}_{low}$ as
\begin{equation} \label{eq:mpnn}
\Psi = \bigcup_{N\geq 1}\{m_N \circ \cdots \circ m_1\ \big|\ m_i \in \mathcal{M}_{up}\cup\mathcal{M}_{low}, 1\leq i\leq N\}.
\end{equation}
We are in fact aiming to show the approximation property of $\Psi$.

\section{Main results}\label{sec:Main results}
Now the main theorem is given as follows, with several conditions required for control families.
\begin{assumption}\label{ass:control_family}
Assume that the control family $\mathcal{NN}^d$ satisfies
\begin{enumerate}
\item For any $f_{net} \in \mathcal{NN}^d$, $f_{net}$ is Lipschitz on any compact set in $\R^d$.

\item For any compact $V\in \R^d$, smooth function $f$ on $V$, and $\varepsilon >0$, there exists $f_{net}\in \mathcal{NN}^d$ such that $\norm{f_{net}-f}_V\leq \varepsilon$.
\end{enumerate}
\end{assumption}
\begin{theorem}\label{thm:main}
Suppose that $D\geq s\geq 2$, $p\in [1,\infty)$, $U\subset \R^D$ is compact, the control families $\mathcal{NN}^{d}$ ($d=D-s+1,s-1$) satisfy Assumption \ref{ass:control_family}, and $\Psi$ is defined as in (\ref{eq:mpnn}). If $\psi:U \rightarrow \R^D$ is measure-preserving, bounded and injective, then
\begin{equation*}
\psi \in \overline{\Psi}_{L^p(U)}.
\end{equation*}
\end{theorem}

Viz., for any $\varepsilon>0$, there exists a measure-preserving neural network $\psi_{net} \in \Psi$ such that
\begin{equation*}
    \norm{\psi-\psi_{net}}_{L^p(U)}\leq \varepsilon.
\end{equation*}
Clearly, there is only identity map in $\Psi$ when dimension $D=1$, thus this conclusion is not true for $D=1$
due to the counterexample $\psi(x)=-x$.

Here, the requirements of map $\psi$, i.e., injection and boundness, are in some sense necessary since the measure-preserving networks are invertible and bounded on compact set. We remark that Theorem \ref{thm:main} also holds if these requirements are not satisfied at countable points due to the $L^p$-norm. In addition, the assumptions for the control family are also necessary for the presented proofs. Fortunately, such conditions are very easy to achieve. Popular activation functions, such as rectified linear unit (ReLU) $ReLU(z) = \max(0, z)$, sigmoid $Sig(z) = 1/(1 + e^{-z})$ and $tanh(z)$, could satisfy the Lipschitz condition; and the well-known universal approximation theorem states that feed-forward networks can approximate essentially any function if their sizes are sufficiently large \citep{cybenko1989approximation,hornik1990universal,shen2021neural}.
The last assumption is also required in the approximation analysis for other structured networks, such as \cite{jin2020sympnets} for SympNets, \cite{zhang2020approximation} for Neural ODEs \citep{chen2018neural} and invertible residual networks \citep{behrmann2019invertible}.

The assumption that $\psi$ is injective, continuously differentiable and $|\det{(J_{\psi})}|=1$ implies that $\psi$ is bounded and measure-preserving due to the transformation formula for integrals. This fact yields the following corollary immediately.
\begin{corollary}\label{cor:1}
Suppose that $D\geq s\geq 2$, $p \in [1,\infty)$, $U\subset \R^D$ is compact, the control families $\mathcal{NN}^{d}$ ($d=D-s+1,s-1$) satisfy Assumption \ref{ass:control_family}, and $\Psi$ is defined as in (\ref{eq:mpnn}). If  $\psi:U \rightarrow \R^D$ is injective, continuously differentiable and $|\det{(J_{\psi})}|=1$, then
\begin{equation*}
\psi \in \overline{\Psi}_{L^p(U)}.
\end{equation*}
\end{corollary}

Finally, we would like to point out that different choices of $s$ in control family lead to same approximation results, thus we use symbols $\Psi$ without emphasizing $s$ (see Sec \ref{sec:properties} for detailed proof). As aforementioned, practical applications including NICE and RevNets could have reshuffling operations and different dimension-splitting mechanisms for each layer. 
If the used hypothesis space contains $\mathcal{M}_{up}$ and $\mathcal{M}_{low}$ for an integer $s$, then it inherits the approximation capabilities.

\section{Discussions}\label{sec:Discussions}
In this section, we will further investigate measure-preserving networks numerically and discuss the potential applications of our results.

\subsection{Learning  measure-preserving flow map}
Measure-preserving of divergence-free dynamical systems is a classical case of geometric structure and is more general than the symplecticity-preserving of Hamiltonian systems. Motivated by the satisfactory works on learning Hamiltonian systems \citep{chen2021data,greydanus2019hamiltonian, jin2020sympnets}, it is also interesting to learn divergence-free dynamics via measure-preserving models. As a by-product, we obtain Lemma \ref{lem:divergence-free} that measure-preserving neural networks are able to approximate arbitrary divergence-free dynamical system. (see Sec \ref{sec:Approximation results for flow maps}).

\begin{figure}[htbp]
    \centering
    \includegraphics[width=0.48\textwidth]{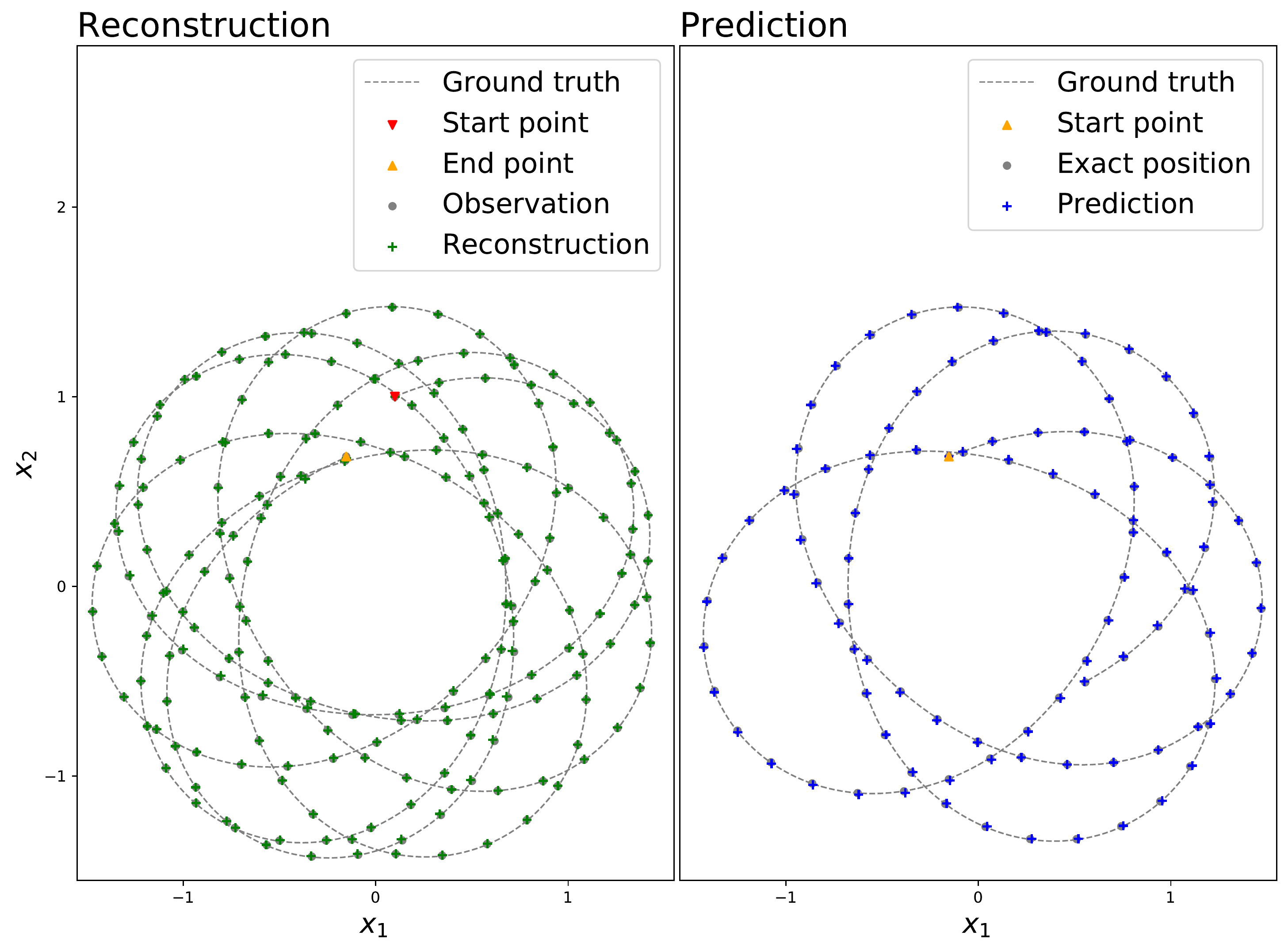}
    \caption{Learning  measure-preserving flow map using measure-preserving neural networks.}
    \label{fig:lf}
\end{figure}

Figure \ref{fig:lf} demonstrates the ability of measure-preserving networks to fit and extrapolate measure-preserving map numerically. Here, the training data $\{(x_n,x_{n+1})\}_{n=0}^{199}$ is obtained by sampling states on a single trajectory of a 4-dimensional divergence-free dynamical system. And we aim to approximate the flow map $\psi$ that maps $x_n$ to $x_{n+1}$ using measure-preserving network $\psi_{net}$. After training, we reconstruct the trajectory and perform predictions for $100$ steps starting at $x_{200}$. All trajectories are projected onto the first two dimensions. More experimental details are shown in \ref{app:Experimental details}. It is observed that the measure-preserving model successfully reconstructs and predicts the evolution of the measure-preserving flow.

\subsection{Application scopes of the approximation theory}
The expected error of neural networks can be divided into three main types: approximation, optimization, and generalization \citep{bottou2007tradeoff,bottou2010large,jin2020quantifying}. See Figure \ref{fig:error_types} for the illustration.
\begin{figure}[htbp]
    \centering
    \includegraphics[width=0.48\textwidth]{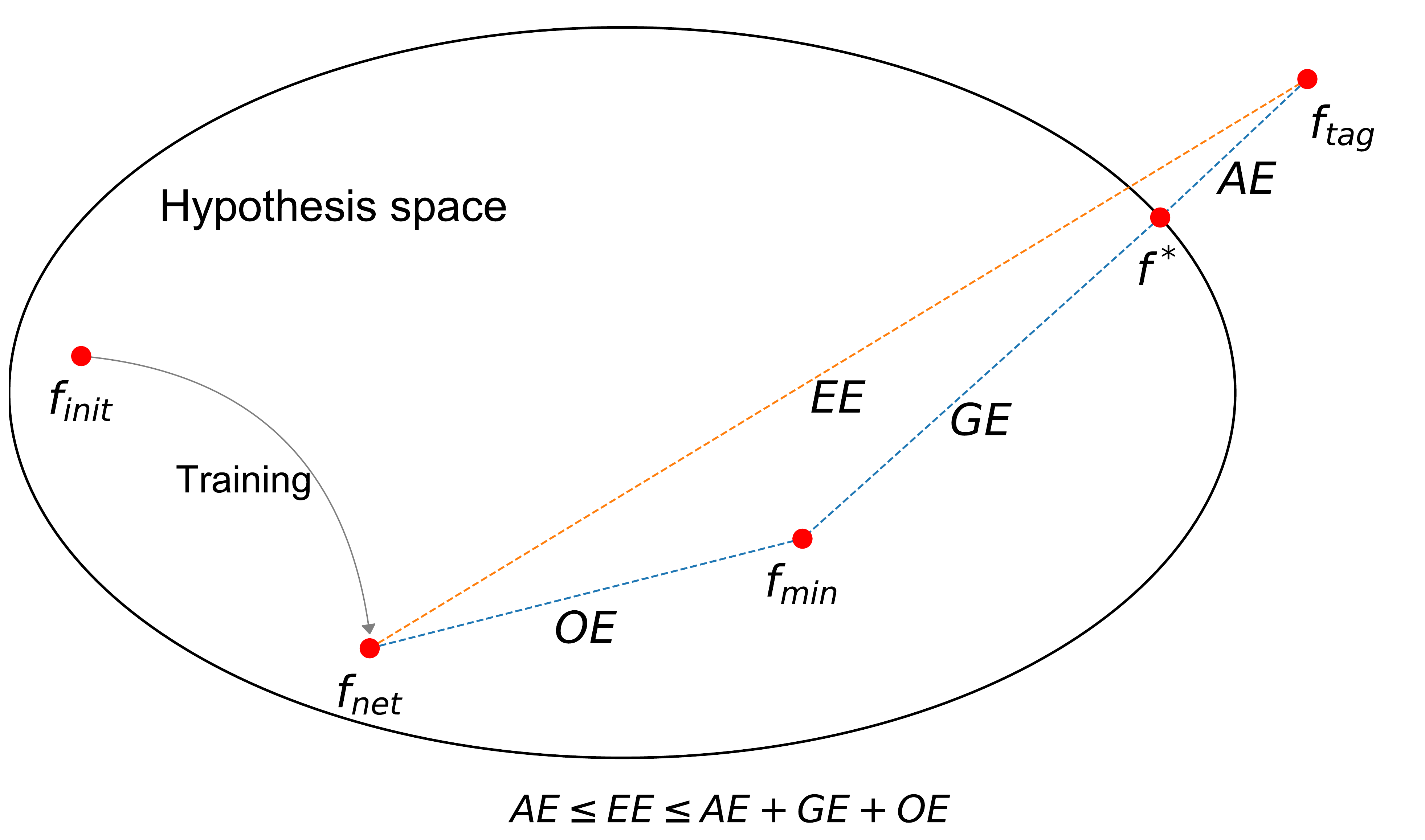}
    \caption{Illustration of optimization error (OE), generalization error (GE), approximation error (AE) and total expected error (EE). Here, $f_{net}$ is the network returned by the training algorithm starting at initial $f_{init}$. $f_{min}$ is the neural network which minimize empirical loss. $f_{tag}$ is the target ground-truth function, and $f^*$ is the network closest to $f_{tag}$ in the hypothesis space.}
    \label{fig:error_types}
\end{figure}

One of the key target in deep learning is to develop algorithms to increase accuracy, while the premise of this purpose is a good upper bound of approximation error. In addition, the approximation error is a crucial part of expected error. For structured deep neural networks, however, the approximation is different from the well-known universal approximation theory \citep{cybenko1989approximation,hornik1990universal} for fully connected networks obtained about 30 years ago and thus is attaching increasing attention (see Sec.\ref{sec:Induction} 3rd paragraph). Here, two application scenarios are important to discuss.

The first is that the target function is speculated to have a specific structure (e.g., CNN for image processing \citep{krizhevsky2012imagenet}, {measure-preserving modules in Poisson networks \citep{jin2020learning}}), or there exists prior knowledge exactly (e.g., HNN for discovery Hamiltonian systems \citep{greydanus2019hamiltonian}, DeepONet for learning nonlinear operators \citep{lu2021learning}, measure-preserving networks for identifying divergence-free dynamics). The approximation theory in this paper indicates the approximation error can be made sufficiently small, which theoretically guarantees the feasibility of measure-preserving network 
modeling measure-preserving map and provides a key ingredient to the error analysis of learning algorithms using measure-preserving models.

The second is that the target function does not involve  structures, but the employed network is designed for certain objectives, such as RevNets for avoiding storing intermediate activation \citep{gomez2017the}, generating models including NICE \citep{dinh2015nice} for computing inverse and Jacobian determinants efficiently, and measure-preserving networks for obtaining exact bijection of lossless compression \citep{zhang2021ivpf}. This compromise of expressivity has a significant impact on performance \citep{dinh2017density,chen2018neural,chen2019residual, huang2018neural,kingma2018glow}. And the approximation results mathematically characterize the limitation of the measure-preserving networks studied in this paper. In addition, our theory indicates that the approximation error mainly depends on the distance between the target function and measure-preserving function space. It would be an interesting future work to quantify this distance although it is not related to neural network theory. One possible approach is polar factorization \citep{brenier1991polar}.

\section{Proofs}\label{sec:Proofs}
Throughout this section we assume that $\Psi$ is defined as in (\ref{eq:mpnn}) with control families $\mathcal{NN}^{d}$ ($d=D-s+1,s-1$) satisfied Assumption \ref{ass:control_family}.

\subsection{Properties of measure-preserving neural networks}\label{sec:properties}
Consider the following auxiliary measure-preserving modules of the form
\begin{equation*}
R^{i}_{K,a,b}(x)= \begin{pmatrix}x[\ :i]\\ x[i] + \hat{\sigma}_{K,a,b}(\overline{x[i]})\\ x[i+1:\ ] \end{pmatrix}
 \end{equation*}
with $1 \leq i \leq D$. Here, $\hat{\sigma}_{K,a,b}$ specifies a fully connected neural network with one hidden layer, i.e.,
\begin{equation*}
\hat{\sigma}_{K,a,b}(\overline{x[i]}) =a \sigma(K\overline{x[i]}+b),\ K \in \R^{W \times (D-1)}, b\in \R^{W \times 1}, a\in \R^{1 \times W},
\end{equation*}
where $\sigma: \R \rightarrow \R$ is the smooth activation function sigmoid $Sig(z) = 1/(1 + e^{-z})$ with Lipschitz constant $L_{\sigma}$. By the universal approximation theorem, $\hat{\sigma}_{K,a,b}$ can approximate any smooth function. 

We denote the collection of $R^{i}_{K,a,b}$ as
\begin{equation*}
    \mathcal{R}^i = \{R^{i}_{K,a,b} |K \in \R^{W \times (D-1)}, b\in \R^{W \times 1}, a\in \R^{1 \times W}\}.
\end{equation*}
Lemma \ref{lem:r1} states that the auxiliary measure-preserving modules defined above can be approximated by measure-preserving neural networks. To prove this claim, we start with the following auxiliary lemma.
\begin{lemma}\label{lem:comerror}
Given a sequence of $\varphi^1, \cdots,\varphi^N$ which map from $\R^D$ to $\R^D$ and are Lipschitz on any compact set. If $\varphi^k \in \overline{\Psi}_{U}$ holds on any compact $U$, $1\leq k \leq N$, then $\varphi^N\circ \cdots \circ \varphi^1 \in \overline{\Psi}_U$ holds on any compact $U$.
\end{lemma}
\begin{proof}
We prove this lemma by induction. To begin with, the case $N=1$ is obvious. Suppose that this lemma holds when $N=n$. For the case $N=n+1$, given compact $U\in\R^D$, define 
\begin{equation*}
V=\bigcup_{k=1}^{n+1}\varphi^k\circ \cdots \circ \varphi^1(U)\cup U    
\end{equation*}
and 
\begin{equation*}
E(V) = \{x \in \R^D\ |\ \exists x'\in V\ \text{s.t.} \norm{x-x'}_{\infty}\leq 1\}, 
\end{equation*}
where $V$ and $E(V)$ are both compact. According to the induction hypothesis, we know that for any $0<\varepsilon <1$ there exists $\phi\in\Psi$ such that
\begin{equation*}
\norm{\varphi^n\circ \cdots \circ \varphi^1-\phi}_{U} \leq \varepsilon.
\end{equation*}
This inequality together with the condition $\varphi^n\circ \cdots \circ \varphi^1(U) \subset V$ yields that $\phi(U)\subset E(V)$. Since $\varphi^{n+1} \in  \overline{\Psi}_{E(V)}$ we can choose $\phi' \in \Psi$ such that
\begin{equation*}
\norm{\varphi^{n+1} - \phi'}_{E(V)} \leq \varepsilon.
\end{equation*}
By the triangle inequality we have
\begin{equation*}
\begin{aligned}
&\norm{\varphi^{n+1}\circ \cdots \circ \varphi^1-\phi'\circ \phi}_{U} \\
\leq& \norm{\varphi^{n+1}\circ \cdots \circ \varphi^1-\varphi^{n+1}\circ\phi}_{U} +\norm{\varphi^{n+1}\circ\phi-\phi'\circ \phi}_{U}\\
\leq& L\norm{\varphi^n\circ \cdots \circ \varphi^1-\phi}_{U} + \norm{\varphi^{n+1} - \phi'}_{E(V)}\\
\leq& (L+1)\varepsilon,
\end{aligned}
\end{equation*}
where $L$ is the Lipschitz constant of $\varphi^{n+1}$ on $E(V)$. Note that $\phi'\circ \phi \in \Psi$, hence
\begin{equation*}
\varphi^{n+1}\circ \cdots \circ \varphi^1 \in \overline{\Psi}_U,
\end{equation*}
which completes the induction.
\end{proof}

\begin{lemma}\label{lem:r1}
$\mathcal{R}^i\subset\overline{\Psi}_U$ for any compact set $U \subset \R^D$, $1 \leq i \leq D$.
\end{lemma}
\begin{proof}
Without loss of generality, we assume $i=1$ and $u = R^1_{K,a,b}$. Taking
\begin{equation*}
\phi^{w}(x)=\begin{pmatrix}x[1] + a[w]\sigma(K[w]x[2:\ ]+b[w])\\ x[2:\ ] \end{pmatrix}
\end{equation*}
for $w = 1,\cdots,W$ yields
\begin{equation*}
u=\phi^{W}\circ \cdots \circ \phi^{1}.
\end{equation*}
It is easy to verify that $\phi^w$ is Lipschitz on any compact set . In order to apply Lemma \ref{lem:comerror}, it suffices to show $\phi^w \in \overline{\Psi}_{V}$ for $w = 1,\cdots,W$ and any compact $V$, we will do it by construction. 

Given any $\epsilon>0$, for $\delta>0$ satisfying $K[w][s-1]+\delta \neq 0$, define
\begin{equation*}
\begin{aligned}
\phi^{w1}(x) &= \begin{pmatrix}x[\ :s]\\ x[s] +(K[w][s-1]+\delta)^{-1}K[w][ \ :s-1]x[2:s] \\ x[s+1:\ ] \end{pmatrix},\\
\phi^{w2}(x) &= \begin{pmatrix} \begin{aligned}x[1] + a[w]\sigma\Big(&(K[w][s-1]+\delta)x[s] \\&+K[w][s:\ ]x[s+1:\ ]+b[w]\Big)\end{aligned}\\ x[2:\ ] \end{pmatrix},\\
\phi^{w3}(x) &= \begin{pmatrix}x[\ :s]\\ x[s] -(K[w][s-1]+\delta)^{-1}K[w][ \ :s-1]x[2:s] \\ x[s+1:\ ] \end{pmatrix}.\\
\end{aligned}
\end{equation*}
We could readily check that
\begin{equation*}
\phi^{w3}\circ\phi^{w2} \circ \phi^{w1} = \begin{pmatrix}x[1] + a[w]\sigma(K[w]x[2:\ ]+b[w]+ \delta x[s])\\ x[2:\ ] \end{pmatrix}.
\end{equation*}
Since
\begin{equation*}
    \norm{\phi^{w3}\circ\phi^{w2} \circ \phi^{w1}-\phi^w}_V\leq C \delta
\end{equation*}
holds for a constant $C$, we can choose a small $\delta$ such that $C\delta<\frac{1}{2}\epsilon$. Furthermore, we have $\phi^{wi}\in \overline{(\mathcal{M}_{up}\cup\mathcal{M}_{low})}_{U'}\subset\overline{\Psi}_{U'}$ for any compact $U'$ according to the second item of Assumpion \ref{ass:control_family}. Applying Lemma \ref{lem:comerror} again and we have $\phi^{w3}\circ\phi^{w2} \circ \phi^{w1}\in\overline{\Psi}_V$, thus there exists $\psi\in\Psi$ such that
\begin{equation*}
    \norm{\psi-\phi^{w3}\circ\phi^{w2} \circ \phi^{w1}}_V\leq \frac{1}{2}\epsilon.
\end{equation*}
Therefore
\begin{equation*}
\begin{aligned}
\norm{\psi-\phi^w}_V&\leq\norm{\psi-\phi^{w3}\circ\phi^{w2} \circ \phi^{w1}}_V+\norm{\phi^{w3}\circ\phi^{w2} \circ \phi^{w1}-\phi^w}_V\\
&\leq\frac{1}{2}\epsilon+\frac{1}{2}\epsilon=\epsilon,
\end{aligned}
\end{equation*}
hence $\phi^w\in\overline{\Psi}_V$.
\end{proof}
The auxiliary modules are also measure-preserving but using special dimension-splitting mechanisms. Clearly, a element in $\mathcal{M}_{up}$ can be written as composition of $s-1$ maps like
\begin{equation*}
R^{i}_{f_{net}}(x)= \begin{pmatrix}x[\ :i]\\ x[i] + f_{net}(\overline{x[i]})\\ x[i+1:\ ] \end{pmatrix}.
 \end{equation*}
This fact together with Lemma \ref{lem:r1} concludes that different choices of $s$ in control family lead to same approximation results theoretically, thus we use  symbol $\Psi$ without emphasizing $s$.

In addition, we show that translation invariance of $\Psi$. This property will be used in Subsection \ref{sec:Proof of Theorem}.
\begin{property}\label{property:translation invariance}
Given $a \in \R^D$. If $\psi_{net} \in \Psi$, then $\psi_{net}+ a\in \overline{\Psi}_U$ for any compact $U$.
\end{property}
\begin{proof}
For any compact $U$, let
\begin{equation*}
\begin{aligned}
&m^{1}: &&\hat{x}[\ :s] = x[\ :s] +a[\ :s],\ \hat{x}[s:\ ] = x[s:\ ];\\
&m^{2}: &&\hat{x}[\ :s] = x[\ :s],\ \hat{x}[s:\ ] = x[s:\ ] + a[s:\ ].\\
\end{aligned}
\end{equation*}
We have $m^{i}\in \overline{(\mathcal{M}_{up}\cup\mathcal{M}_{low})}_{U'}\subset\overline{\Psi}_{U'}, i=1,2$ for any compact $U'$ according to the second item of Assumpion \ref{ass:control_family}. By Lemma \ref{lem:comerror} we know
\begin{equation*}
    m^2 \circ m^1 \circ \psi_{net} \in \overline{\Psi}_U,
\end{equation*}
which concludes the property.
\end{proof}

\subsection{Approximation results for flow maps}\label{sec:Approximation results for flow maps}
Recently, the dynamical systems approach led to much progress in the theoretical underpinnings of deep learning \citep{e2017proposal,e2019mean, li2017maximum}. In particular, \cite{li2020deep} build approximation theory for continuous-time deep residual neural  networks. These developments inspire us to apply differential equation techniques to complete the proof. The results of this work also serves the effectiveness of the dynamic system approach for understanding deep learning. Consider a differential equation
\begin{equation}\label{eq:ODE}
\frac{d}{dt}y(t) = f(t, y(t)),\quad y(\tau)=x, \tau\geq 0,
\end{equation}
where $y(t) \in \R^D$, $f: [0,+\infty)\times \R^D \rightarrow \R^D$ is smooth. For a given time step $T\geq 0$, $y(\tau+T)$ could be regarded as a function of its initial condition $x$. We denote $\varphi_{\tau,T,f}(x) := y(\tau+T)$, which is known as the time-$T$ flow map of the dynamical system (\ref{eq:ODE}). We also write the collection of such flow maps as
\begin{equation*}
    \mathcal{F}(U) = \left\{\varphi_{\tau, T, f}: U \rightarrow \R^D \ \big|\ \tau,T\geq0,\ f \in C^{\infty}([0,+\infty)\times \R^D) \right\}.
\end{equation*}

Following \citep{hairer1993solving,hairer2006geometric}, we briefly recall some essential supporting results of numerical integrators here.
\begin{definition}
Given system (\ref{eq:ODE}), an integrator $\Phi_{\tau,h,f}$ with time step $h$ has order $p$, if for compact $U \subset \R^D$, and any $\tau'$ in a compact time interval, there exists constant $C$ such that for sufficiently small step $h>0$,
\begin{equation*}
\norm{\Phi_{\tau',h,f}-\varphi_{\tau',h,f}}_{U}\leq C h^{p+1}.
\end{equation*}
\end{definition}
The order of integrator is usually pointwise defined in the literature. Here $U$ is compact and thus the above definition accords with the literature. The simplest numerical integrator is the explicit Euler method,
\begin{equation*}
\Phi_{\tau,h,f}^e(x) = x + h f(\tau,x).
\end{equation*}
Another scheme will be used in this paper is a splitting method. For system (\ref{eq:ODE}), if $f=\sum_{k=1}^K f^k$, the formula is given as
\begin{equation*}
\Phi_{\tau,h,f}^s(x) = \varphi_{\tau,h,f^K}\circ \cdots \circ \varphi_{\tau, h,f^1}(x).
\end{equation*}
The above numerical integrators are both of order $1$.

Next, we turn to the approximation aspects of measure-preserving flow maps. Measure-preserving is a certain geometric structure of continuous dynamical systems. As demonstrated in \cite[Section  \uppercase\expandafter{\romannumeral6}.6]{hairer2006geometric}, measure is preserved by the flow of differential equations with a divergence-free vector field.
\begin{proposition}\label{pro:divergence free}
The flow map of system (\ref{eq:ODE}) is measure-preserving if and only if
\begin{equation*}
\dive_y f= \sum_{d=1}^D \frac{\partial f_d}{\partial y_d}=0,
\end{equation*}
where $f=(f_1, \cdots, f_D)^\top$, $y=(y_1, \cdots, y_D)^\top$.
\end{proposition}
By Proposition \ref{pro:divergence free}, we denote the set of measure-preserving flow maps as
\begin{equation*}
    \mathcal{VF}(U) = \left\{\varphi_{\tau, T, f}\in \mathcal{F}(U)\ \big|\ \dive_y f=0\right\}.
\end{equation*}
Subsequently, we introduce two kinds of vector fields of measure-preserving flow maps.
\begin{definition}
For $f: [0,+\infty)\times \R^D \rightarrow \R^D$ and $1\leq d \leq D-1$, we say $f$ is $2$-Hamiltonian in the $d,d+1$-th variables if there exists a scalar function $H: [0,+\infty)\times\R^D \rightarrow \R$ such that
\begin{equation*}
f = (\underbrace{0,\cdots,0}_{d-1},-\frac{\partial H}{\partial y_{d+1}}, \frac{\partial H}{\partial y_d},\underbrace{0,\cdots,0}_{D-d-1})^\top.
\end{equation*}
\end{definition}

\begin{definition}
For $f: [0,+\infty)\times \R^D \rightarrow \R^D$ and $1\leq d \leq D-1$, we say $f$ is separable $2$-Hamiltonian in the $d,d+1$-th variables if there exist two scalar functions $g_1,g_2:  [0,+\infty)\times\R^{D-1} \rightarrow \R$  such that
\begin{equation*}
f = (\underbrace{0,\cdots,0}_{d-1},g_1(t,\overline{y[d]}), g_2(t,\overline{y[d+1]}),\underbrace{0,\cdots,0}_{D-d-1})^\top.
\end{equation*}
\end{definition}

Clearly, a separable $2$-Hamiltonian $f$ is $2$-Hamiltonian and both are divergence-free. Below we will establish the approximation results for flow maps with separable 2-Hamiltonian vector fields (Lemma \ref{lem:one-Hami}) and 2-Hamiltonian vector fields (Lemma \ref{lem:two-Hami}), and finally obtain the approximation theory of measure-preserving flow maps (Lemma \ref{lem:divergence-free}).

To this end, we present the composition approximation for flow map firstly, which will be used frequently.

\begin{lemma}\label{lem:diserror}
Given smooth $f: [0,+\infty)\times \R^D \rightarrow \R^D$ and $\varphi_{\tau, T,f} \in \mathcal{F}(U)$ with compact set $U \subset \R^D$.
If on any compact $U'$, there exists $\phi\in \Psi$ such that
\begin{equation*}
\norm{\varphi_{\tau',h,f} - \phi}_{U'} \leq Ch^2
\end{equation*}
holds for any $\tau' \in [\tau,\tau+T]$ and any sufficiently small step $h>0$,
then,
\begin{equation*}
\varphi_{\tau,T,f} \in \overline{\Psi}_U.
\end{equation*}
\end{lemma}
\begin{proof}
Define
\begin{equation*}
V = \{\varphi_{\tau',T', f}(x)\ |\ x\in U,\ \tau\leq\tau'\leq \tau'+T' \leq \tau+T\},
\end{equation*}
and for $i=1,2$,
\begin{equation*}
E^i(V) = \{x \in \R^D\ |\ \exists x'\in V\ \text{s.t.} \norm{x-x'}_{\infty}\leq i\},
\end{equation*}
where $V$ and $E^i(V)$ are compact since $f$ is smooth.
Let
\begin{equation*}
L =1+ \sup\limits_{\substack{\tau\leq\tau'\leq \tau'+T' \leq \tau+T\\x \in {E^2(V)}}} \norm{\frac{\partial \varphi_{\tau',T', f}(x)}{\partial x}}_{\infty}.
\end{equation*}
And for any $0<\varepsilon <1$, take
\begin{equation*}
    N > \frac{1}{\varepsilon}(LCT^2 + T\norm{f}_{[\tau,\tau+T]\times E^2(V)}), \quad h=\frac{T}{N}.
\end{equation*}
Then there exists a sequence of $\phi_{N-1}, \cdots, \phi_{0} \in \Psi$, such that, for $0 \leq k\leq N-1$,
\begin{equation*}
\norm{\varphi_{\tau+kh,h,f} - \phi_k}_{E^1(V)} \leq Ch^2.
\end{equation*}
To conclude the lemma, it suffices to show that
\begin{equation*}
\norm{\varphi_{\tau,nh,f}-\phi_{n-1}\circ \cdots \circ \phi_{0}}_{U} \leq n\cdot L\cdot C\cdot \frac{T^2}{N^2}
\end{equation*}
for any $1\leq n \leq N$.
We now prove this statement by induction on $1\leq n \leq N$. First, the case when $n=1$ is obvious. Suppose now
\begin{equation*}
\norm{\varphi_{\tau,kh,f}-\phi_{k-1}\circ \cdots \circ \phi_{0}}_{U} \leq k\cdot L\cdot C\cdot \frac{T^2}{N^2}
\end{equation*}
for $k\leq n-1$. This inductive
hypothesis implies $\phi_{k-1}\circ \cdots \circ \phi_{0}(U) \subset E^1(V)$ and thus
\begin{equation*}
\begin{aligned}
&\phi_{k}\circ \phi_{k-1}\circ \cdots \circ \phi_{0}(U) \subset E^2(V),\\
&\varphi_{\tau+kh, h,f}\circ \phi_{k-1}\circ \cdots \circ \phi_{0}(U) \subset E^2(V),\\
\end{aligned}
\end{equation*}
where we have used the fact that for any $x\in E^1(V)$,
\begin{equation*}
\begin{aligned}
\norm{\varphi_{\tau+kh, h,f}(x)  - x}_{\infty} &= \norm{\int_{\tau+kh}^{\tau+(k+1)h}f(t,x(t))dt}_{\infty}\\
&\leq h \norm{f}_{[\tau,\tau+T]\times E^2(V)}\leq \varepsilon < 1,
\end{aligned}
\end{equation*}
and
\begin{equation*}
\begin{aligned}
\norm{\phi_{k}(x)-x}_{\infty}&\leq \norm{\varphi_{\tau+kh, h,f}(x)  - x}_{\infty} + Ch^2 \\
&\leq h \norm{f}_{[\tau,\tau+T]\times E^2(V)}+ Ch^2\leq \varepsilon < 1.
\end{aligned}
\end{equation*}
Subsequently, denote $L_k = \sup_{x \in E^2(V)}\norm{\frac{\partial \varphi_{\tau+(k+1)h,(n-k-1)h,f}(x)}{\partial x}}_{\infty} \leq L$, we obtain
\begin{equation*}
\begin{aligned}
&\norm{\varphi_{\tau,nh,f}-\phi_{n-1}\circ \cdots \circ \phi_{0}}_{U} \\
\leq &\Big(\sum_{k=1}^{n-1}\|\varphi_{\tau+kh,(n-k)h,f}\circ \phi_{k-1}\circ \cdots \circ \phi_{0} \\ & \quad - \varphi_{\tau+(k+1)h,(n-k-1)h,f}\circ \phi_{k}\circ \cdots \circ \phi_{0}\|_{U}\Big) \\
&+\norm{\varphi_{\tau,nh,f} - \varphi_{\tau+h,(n-1)h,f} \circ \phi_{0}}_{U}\\
\leq &\sum_{k=1}^{n-1} L_k \norm{\varphi_{\tau+kh, h,f}\circ \phi_{k-1}\circ \cdots \circ \phi_{0} -  \phi_{k}\circ \cdots \circ \phi_{0}}_{U}\\
&+L_0 \norm{\varphi_{\tau,h,f}-\phi_{0}}_{U}\\
\leq &\sum_{k=1}^{n-1} L \norm{\varphi_{\tau+kh, h,f}-  \phi_{k}}_{E^1(V)} + LC \frac{T^2}{N^2}\\
\leq &n\cdot L\cdot C\cdot \frac{T^2}{N^2}.
\end{aligned}
\end{equation*}
Hence the induction holds and the proof is completed.
\end{proof}

\begin{lemma}\label{lem:one-Hami}
Given compact $U\subset \R^D$ and $\varphi_{\tau, T, f}\in \mathcal{F}(U)$. If the vector fields $f$ is separable $2$-Hamiltonian in the $d,d+1$-th variables with $1 \leq d \leq D-1$, then,
\begin{equation*}
\varphi_{\tau,T,f}\in \overline{\Psi}_{U}.
\end{equation*}
\end{lemma}
\begin{proof}
Without loss of generality, we assume $d=1$. The relation between $x$ and $\varphi_{\tau,T,f}(x)$ is characterized by the following equation,
\begin{equation}\label{eq:1sode}
\begin{aligned}
\frac{d}{dt}y(t)= f(t,y(t)), \quad y(\tau)=x,\ y(\tau+T) = \varphi_{\tau,T,f}(x).
\end{aligned}
\end{equation}
For $y = (y_1,y_2,y_3,\cdots, y_d)$, denote $p=y_1, q=y_2, \mu=(y_3,\cdots, y_d)^\top$. Since $f$ is separable $2$-Hamiltonian in the $1,2$-th variables, there exist two scalar functions $g_1,g_2:  [0,+\infty)\times\R^{D-1} \rightarrow \R$ such that equation (\ref{eq:1sode}) can be written as
\begin{equation}\label{eq:1shami}
\begin{aligned}
\frac{d}{dt}p(t) &= g_1(t,q,\mu), \\
\frac{d}{dt}q(t) &= g_2(t,p,\mu), \\
\frac{d}{dt}\mu(t) &=0.
\end{aligned}
\end{equation}
For any $\tau' \in [\tau,\tau+T]$ and any sufficiently small step $h>0$, define the following map
\begin{equation*}
\begin{aligned}
\phi_{\tau',h}^1(p,q,\mu) &= (p + hg_1(\tau',q,\mu), q, \mu^\top)^\top,\\
\phi_{\tau', h}^2(p,q,\mu) &= (p, q+hg_2(\tau',p,\mu),\mu^\top)^\top,\\
\phi_{\tau', h} &= \phi_{\tau', h}^2 \circ \phi_{\tau', h}^1.
\end{aligned}
\end{equation*}
Here, $\phi_{\tau', h}$ is the splitting integrator applied to system (\ref{eq:1shami}), which is an integrator of order one. Therefore, for any compact $U'$, there exists constant $C$ such that
\begin{equation*}
\norm{\varphi_{\tau',h,f} - \phi_{\tau',h}}_{U'}\leq Ch^2.
\end{equation*}

In addition, for any compact $V$, the universal approximation theorem of neural networks with one hidden layer and sigmoid activation together with Lemma \ref{lem:r1} implies
\begin{equation*}
\begin{aligned}
&\phi_{\tau',h}^1 \in \overline{\mathcal{R}^1}_{V} \subset \overline{\Psi}_{V},\\
&\phi_{\tau',h}^2 \in \overline{\mathcal{R}^2}_{V} \subset \overline{\Psi}_{V}.\\
\end{aligned}
\end{equation*}
By Lemma \ref{lem:comerror}, we obtain $\phi_{\tau',h} \in \overline{\Psi}_{U'}$ and thus there exists $v \in \Psi$ such that
\begin{equation*}
\norm{v - \phi_{\tau',h}}_{U'}\leq h^2.
\end{equation*}

Finally, we conclude that
\begin{equation*}
\norm{\varphi_{\tau',h,f} - v}_{U'}\leq (C+1)h^2,
\end{equation*}
and the lemma is completed by applying Lemma \ref{lem:diserror}.
\end{proof}

\begin{proposition}\label{lem:approximation hamiltonian}
Given any non-autonomous $H(t,p,q,\mu)$ with bounded parameter $\mu$, polynomial in $p, q\in \R$, and the Hamiltonian system
\begin{equation*}
\begin{aligned}
&\frac{d}{dt}p(t) = -\frac{\partial H}{\partial q}(t,p,q,\mu), \\
&\frac{d}{dt}q(t) =  \frac{\partial H}{\partial p}(t,p,q,\mu). \\
\end{aligned}
\end{equation*}
Denote $f_1 = (-\frac{\partial H}{\partial q}, \frac{\partial H}{\partial p})^\top$. Then on any
compact domain $U$ in the $(p,q,\mu)$-space and any compact interval of the values of $\tau$, there exists a scalar function $V(t,q,\mu)$ polynomial in $q$, such that, for any sufficiently small step $h>0$, the time-$2\pi$ flow map of the Hamiltonian system
\begin{equation*}
\begin{aligned}
&\frac{d}{dt}p(t) = -q-h\cdot \frac{\partial V}{\partial q}(t,q,\mu), \\
&\frac{d}{dt}q(t) =  p,\\
\end{aligned}
\end{equation*}
denoted as $\varphi_{0,2\pi,f_2}$ with $f_2=(-q-h \frac{\partial V}{\partial q}(t,q,\mu),p)^\top$, satisfies
\begin{equation*}
\sup_{(p,q,\mu) \in U}\norm{\varphi_{0,2\pi,f_2}(p,q) - \varphi_{\tau, h, f_1}(p,q)}_{\infty}\leq C h^2
\end{equation*}
with constant $C$.
\end{proposition}
\begin{proof}
The proposition is the 2-dimensional case of \citep[ Lemma 1]{turaev2002polynomial}.
\end{proof}
With Proposition \ref{lem:approximation hamiltonian}, we can approximate the flow maps with 2-Hamiltonian vector fields, which give rise to the following lemma.

\begin{lemma}\label{lem:two-Hami}
Given compact $U\subset \R^D$ and $\varphi_{\tau, T, f}\in \mathcal{F}(U)$. If the vector fields $f$ is $2$-Hamiltonian in the $d,d+1$-th variables with $1 \leq d \leq D-1$, then,
\begin{equation*}
\varphi_{\tau,T,f}\in \overline{\Psi}_{U}.
\end{equation*}
\end{lemma}
\begin{proof}
Without loss of generality, we assume $d=1$. The relation between $x$ and $\varphi_{\tau,T,f}(x)$ is characterized by the following equation,
\begin{equation}\label{eq:2sode}
\begin{aligned}
\frac{d}{dt}y(t)= f(t,y(t)), \quad y(\tau)=x,\ y(\tau+T) = \varphi_{\tau,T,f}(x).
\end{aligned}
\end{equation}
For $y = (y_1,y_2,y_3,\cdots, y_d)$, denote $p=y_1, q=y_2, \mu=(y_3,\cdots, y_d)^\top$. Since $f$ is $2$-Hamiltonian in the $1,2$-th variables, there exists a scalar function $H:  [0,+\infty)\times\R^D \rightarrow \R$ such that equation (\ref{eq:2sode}) can be written as
\begin{equation*}
\begin{aligned}
\frac{d}{dt}p(t) &= -\frac{\partial H}{\partial q}(t,p,q,\mu), \\
\frac{d}{dt}q(t) &=  \frac{\partial H}{\partial p}(t,p,q,\mu), \\
\frac{d}{dt}\mu(t) &=0.
\end{aligned}
\end{equation*}

On any compact $U'$, since polynomials are dense among smooth functions, for any sufficiently small step $h >0$, there exists $H_{poly}$, polynomial in $p,q$, such that
\begin{equation*}
\norm{\frac{\partial H}{\partial q}-\frac{\partial H_{ploy}}{\partial q}}_{[\tau, \tau+T]\times U'} +\norm{\frac{\partial H}{\partial p}-\frac{\partial H_{ploy}}{\partial p}}_{[\tau, \tau+T]\times U'}\leq h.
\end{equation*}
Consider the Hamiltonian system with Hamiltonian $H_{ploy}$, i.e.,
\begin{equation}\label{eq:2shami1}
\begin{aligned}
\frac{d}{dt}p(t) &= -\frac{\partial H_{ploy}}{\partial q}(t,p,q,\mu), \\
\frac{d}{dt}q(t) &=  \frac{\partial H_{ploy}}{\partial p}(t,p,q,\mu), \\
\frac{d}{dt}\mu(t) &=0.
\end{aligned}
\end{equation}
Denote $f_1 = (-\frac{\partial H_{ploy}}{\partial q}, \frac{\partial H_{ploy}}{\partial p}, 0)^\top$, for $\tau'\in [\tau, \tau+T]$, the time-$h$ flow map of (\ref{eq:2shami1}) starting at $\tau'$ can be written as $\varphi_{\tau',h,f_1}$. Due to the difference between $f$ and $f_1$, there is a constant $C_1$ such that
\begin{equation*}
\begin{aligned}
&\norm{\varphi_{\tau',h,f_1} - \varphi_{\tau',h,f}}_{U'}\\
\leq& \norm{\varphi_{\tau',h,f_1} - \Phi_{\tau',h,f_1}^e}_{U'}+
\norm{\Phi_{\tau',h,f_1}^e - \Phi_{\tau',h,f}^e}_{U'}+
\norm{\Phi_{\tau',h,f}^e - \varphi_{\tau',h,f}}_{U'}\\
\leq& C_1h^2.
\end{aligned}
\end{equation*}

According to Proposition \ref{lem:approximation hamiltonian}, there exists a function $V(t,q,\mu)$ polynomial in $q$ and a Hamiltonian system of the form
\begin{equation}\label{eq:2shami2}
\begin{aligned}
\frac{d}{dt}p(t) &= -q-h \cdot \frac{\partial V}{\partial q}(t,q,\mu), \\
\frac{d}{dt}q(t) &=  p,\\
\frac{d}{dt}\mu(t) &=0,
\end{aligned}
\end{equation}
such that,  the time-$2\pi$ map of (\ref{eq:2shami2}), denoted as $\varphi_{0,2\pi,f_2}$ with $f_2=(-q-h \frac{\partial V}{\partial q}(t,q,\mu),p,0)^\top$, satisfies
\begin{equation*}
\norm{\varphi_{0,2\pi,f_2} - \varphi_{\tau',h, f_1}}_{U'}\leq C_2 h^2
\end{equation*}
with constant $C_2$. Hence,
\begin{equation*}
\begin{aligned}
\norm{\varphi_{0,2\pi,f_2} - \varphi_{\tau',h, f}}_{U'}&\leq \norm{\varphi_{0,2\pi,f_2} - \varphi_{\tau',h, f_1}}_{U'}+\norm{\varphi_{\tau',h,f_1} - \varphi_{\tau',h, f}}_{U'}\\&\leq (C_1+C_2) h^2.
\end{aligned}
\end{equation*}
Subsequently, by Lemma \ref{lem:one-Hami}, there exists $v \in \Psi$ such that
\begin{equation*}
\begin{aligned}
\norm{v - \varphi_{\tau',h, f}}_{U'}&\leq \norm{\varphi_{0,2\pi,f_2} - v}_{U'}+\norm{\varphi_{\tau',h,f_2} - \varphi_{\tau',h, f}}_{U'}\\&\leq (C_1+C_2+1) h^2.
\end{aligned}
\end{equation*}
The lemma is completed as a consequence of Lemma \ref{lem:diserror}.
\end{proof}

\begin{proposition}\label{thm:dfode}
If $f: [0, +\infty)\times\R^D \rightarrow \R^D$ obeys $\dive f =0$, then $f$ can be written as the sum of $D-1$ vector fields
\begin{equation*}
f=f_{1,2}+f_{2,3}+ \cdots +f_{D-1,D},
\end{equation*}
where each $f_{d,d+1}$ is $2$-Hamiltonian in the $d,d+1$-th variables for $1\leq d \leq D-1$. Furthermore, if $f$ is smooth, $f_{d,d+1}$ is smooth.
\end{proposition}
\begin{proof}
The proof can be found in \citep{feng1995volume}.
\end{proof}
Proposition \ref{thm:dfode} is founded by Feng and Shang to develop integrator for divergence-free equations. With the decomposition of Proposition \ref{thm:dfode}, the gap between divergence-free and 2-Hamiltonian vector fields is bridged.
\begin{lemma}\label{lem:divergence-free}
Given compact $U\subset \R^D$ and $\varphi_{\tau, T, f}\in \mathcal{VF}(U)$, then,
\begin{equation*}
\varphi_{\tau,T,f}\in \overline{\Psi}_{U}.
\end{equation*}
Viz., $\mathcal{VF}(U) \subset \overline{\Psi}_{U}$.
\end{lemma}

\begin{proof}
By Proposition \ref{thm:dfode}, $f$ can be written as the sum of $D-1$ vector fields
\begin{equation*}
f=f_{1,2}+f_{2,3}+ \cdots +f_{D-1,D},
\end{equation*}
where each $f_{d,d+1}$ is $2$-Hamiltonian in the $d,d+1$-th variables. For any compact set $U' \subset \R^D$, any $\tau' \in [\tau, \tau+T]$ and any sufficiently small step $h> 0$, taking the splitting integrator
\begin{equation*}
\phi_{\tau',h} = \varphi_{\tau',h,f_{D-1,D}}\circ \cdots \circ \varphi_{\tau', h,f_{1,2}}
\end{equation*}
implies
\begin{equation*}
\norm{\varphi_{\tau', h, f} - \phi_{\tau',h}}_{U'} \leq Ch^2.
\end{equation*}

In addition, for any compact $V$, due to Lemma \ref{lem:two-Hami}, we have
\begin{equation*}
\varphi_{\tau', h,f_{d, d+1}} \in \overline{\Psi}_{V},
\end{equation*}
which implies 
$\phi_{\tau',h} \in \overline{\Psi}_{U'}$ according to Lemma \ref{lem:comerror}. Therefore there exists $v \in \Psi$ such that
\begin{equation*}
\norm{\varphi_{\tau', h, f} - v}_{U'} \leq \norm{v - \phi_{\tau',h}}_{U'} + \norm{\varphi_{\tau', h, f} - \phi_{\tau',h}}_{U'} \leq (C+1)h^2.
\end{equation*}
By Lemma \ref{lem:diserror}, we obtain
\begin{equation*}
\varphi_{\tau,T,f} \in \overline{\Psi}_{U},
\end{equation*}
which concludes the proof.
\end{proof}

\subsection{Proof of Theorem \ref{thm:main}}\label{sec:Proof of Theorem}
\begin{proposition}\label{lem:map-flow}
Suppose that $Q \in \R^D$ is an open cube and that $1 \leq p < + \infty$. For every measure-preserving map $\psi:\overline{Q}\rightarrow \overline{Q}$ and arbitrary $\varepsilon>0$, there exists a time-$1$ flow map $\varphi_{0,1,f} \in \mathcal{VF}(Q)$ where $f$ is compactly supported in $(0,1) \times Q$ such that
\begin{equation*}
    \norm{\psi - \varphi_{0,1,f}}_{L^p(Q)}\leq \varepsilon.
\end{equation*}
\end{proposition}
\begin{proof}
The proof can be found in \citep[Corollary 1.1]{brenier2003p}.
\end{proof}

With these results, we are able to provide the proof of the main theorems.
\begin{proof}[Proof of Theorem \ref{thm:main}]
For compact $U \subset \R^D$, we can take $a \in \R^D$ satisfying $U \cap( \psi (U)+a) = \varnothing$. Let $Q$ be a open cube large enough such that $U,  \psi (U)+a \subset Q$, and define $\tilde{\psi}$ on $Q$ by
\begin{equation*}
\tilde{\psi}(x) = \left\{
\begin{aligned}
&\psi(x)+a, \quad &&\text{if } x \in U,\\
&\psi^{-1}(x-a), &&\text{if } x \in \psi(U)+a ,\\
&x, \quad &&\text{if } x \in Q\setminus (U \cup (\psi(U)+a)).
\end{aligned}\right.
\end{equation*}
Here, $\tilde{\psi}:Q\rightarrow Q$ is measure-preserving. According to Proposition \ref{lem:map-flow} there exists a time-$1$ flow map $\varphi_{0,1,f} \in \mathcal{VF}(Q)$ such that
\begin{equation*}
    \norm{\tilde{\psi} - \varphi_{0,1,f}}_{L^p(Q)}\leq \varepsilon,
\end{equation*}
and $f$ is compactly supported in $(0,1)\times Q$. Using Lemma \ref{lem:divergence-free} we deduce that there exists a measure-preserving neural network $\psi_{net} \in \Psi$ such that
\begin{equation*}
    \norm{\varphi_{0,1,f}-\psi_{net}}_{L^p(Q)}\leq \varepsilon.
\end{equation*}
By these estimations, we obtain
\begin{equation*}
    \norm{\psi_{net}-\tilde{\psi}}_{L^p(U)} = \norm{\psi_{net}-a-\psi}_{L^p(U)}\leq 2\varepsilon,
\end{equation*}
and thus $\psi \in\overline{\Psi}_{L^p(U)}$ since $\psi_{net}-a\in \overline{\Psi}_U$. Hence, the theorem has been completed.
\end{proof}

\section{Summary}\label{sec:Summary}
The main contribution of this paper is to prove the approximation capabilities of measure-preserving neural networks. These results serve the mathematical foundations of existing measure-preserving neural networks such as NICE \citep{dinh2015nice} and RevNets \citep{gomez2017the}.

The key idea is introducing flow maps from the perspective of dynamical systems. Via investigation of approximation aspects of two special measure-preserving maps, i.e, flow maps of 2-Hamiltonian and separable 2-Hamiltonian vector fields, we show that every measure-preserving map can be approximated in $C$-norm by measure-preserving neural networks. Finally, by the $L^p$-norm approximation 
proposition which connects measure-preserving flow maps and general measure-preserving maps, we conclude the main theorem.

One open question is the $C$-norm approximation of Corollary \ref{cor:1}. This issue is essentially the gap between measure-preserving flow map and general measure-preserving map. We conjecture that Proposition \ref{lem:map-flow} can be further improved to provide $C$-norm approximation under additional assumptions of measure-preserving map. This paper also shows the effectiveness of understanding deep learning via dynamical systems. Exploring approximation aspects of other structured neural networks
via flow map might be another interesting direction.

\section*{Acknowledgments}
This research is supported by the Major Project on New Generation of Artificial Intelligence from MOST of China (Grant No. 2018AAA0101002), and National Natural Science Foundation of China (Grant No. 12171466).

\appendix

\section{Experimental details} \label{app:Experimental details}
We consider a divergence-free dynamical system given as 
\begin{equation*}
\begin{aligned}
\dot{y}_1 &= y_3,\\
\dot{y}_2 &= y_4,\\
\dot{y}_3 &= \frac{y_1}{100(y_1^2+y_2^2)^{\frac{3}{2}}}+ (y_1^2+y_2^2)^{\frac{1}{2}}y_4,\\
\dot{y}_4 &= \frac{x_2}{100(y_1^2+y_2^2)^{\frac{3}{2}}}- (y_1^2+y_2^2)^{\frac{1}{2}}y_3.
\end{aligned}
\end{equation*}
This equation describes dynamics of a single charged particle in an electromagnetic field governed by Lorentz force. We can readily check that the governing function is divergence-free and thus its flow map is measure-preserving due to Proposition \ref{pro:divergence free}. The architecture used is a stack of 8 coupling layers with partition $s=2$, where single hidden layer neural network with width of $64$ and sigmoid activation is adopted as control families. We optimize the mean-squared-error loss
\begin{equation*}
\frac{1}{I}\sum_{n=1}^{N} \|\psi_{net} (x_n) - x_{n+1}\|^2
\end{equation*}
for $8\times 10^5$ epochs with Adam optimization and learning rate $0.001$. Here, $\{(x_n,x_{n+1})\}_{n=0}^{N}$ is the training data with $N=199$ and is sampled on the trajectory starting at $(0.1,1,1.1,0.5)$ from $t = 0$ to $t = 40$ using equidistant time step size of $0.2$.

\bibliographystyle{elsarticle-harv}
\bibliography{main}

\begin{thebibliography}{43}
\expandafter\ifx\csname natexlab\endcsname\relax\def\natexlab#1{#1}\fi
\providecommand{\url}[1]{\texttt{#1}}
\providecommand{\href}[2]{#2}
\providecommand{\path}[1]{#1}
\providecommand{\DOIprefix}{doi:}
\providecommand{\ArXivprefix}{arXiv:}
\providecommand{\URLprefix}{URL: }
\providecommand{\Pubmedprefix}{pmid:}
\providecommand{\doi}[1]{\href{http://dx.doi.org/#1}{\path{#1}}}
\providecommand{\Pubmed}[1]{\href{pmid:#1}{\path{#1}}}
\providecommand{\bibinfo}[2]{#2}
\ifx\xfnm\relax \def\xfnm[#1]{\unskip,\space#1}\fi
\bibitem[{Behrmann et~al.(2019)Behrmann, Grathwohl, Chen, Duvenaud and
  Jacobsen}]{behrmann2019invertible}
\bibinfo{author}{Behrmann, J.}, \bibinfo{author}{Grathwohl, W.},
  \bibinfo{author}{Chen, R.T.Q.}, \bibinfo{author}{Duvenaud, D.},
  \bibinfo{author}{Jacobsen, J.H.}, \bibinfo{year}{2019}.
\newblock \bibinfo{title}{Invertible residual networks}, in:
  \bibinfo{booktitle}{Proceedings of the 36th International Conference on
  Machine Learning}, \bibinfo{publisher}{PMLR}, \bibinfo{address}{Long Beach,
  California, USA}. pp. \bibinfo{pages}{573--582}.
\bibitem[{Bottou(2010)}]{bottou2010large}
\bibinfo{author}{Bottou, L.}, \bibinfo{year}{2010}.
\newblock \bibinfo{title}{Large-scale machine learning with stochastic gradient
  descent}, in: \bibinfo{booktitle}{Proceedings of COMPSTAT'2010},
  \bibinfo{publisher}{Physica-Verlag HD}, \bibinfo{address}{Heidelberg}. pp.
  \bibinfo{pages}{177--186}.
\bibitem[{Bottou and Bousquet(2007)}]{bottou2007tradeoff}
\bibinfo{author}{Bottou, L.}, \bibinfo{author}{Bousquet, O.},
  \bibinfo{year}{2007}.
\newblock \bibinfo{title}{The tradeoffs of large scale learning}, in:
  \bibinfo{booktitle}{Advances in Neural Information Processing Systems 20,
  Proceedings of the Twenty-First Annual Conference on Neural Information
  Processing Systems, Vancouver, British Columbia, Canada, December 3-6, 2007},
  \bibinfo{publisher}{Curran Associates, Inc.}. pp. \bibinfo{pages}{161--168}.
\bibitem[{Brenier(1991)}]{brenier1991polar}
\bibinfo{author}{Brenier, Y.}, \bibinfo{year}{1991}.
\newblock \bibinfo{title}{Polar factorization and monotone rearrangement of
  vector-valued functions}.
\newblock \bibinfo{journal}{Communications on pure and applied mathematics}
  \bibinfo{volume}{44}, \bibinfo{pages}{375--417}.
\bibitem[{Brenier and Gangbo(2003)}]{brenier2003p}
\bibinfo{author}{Brenier, Y.}, \bibinfo{author}{Gangbo, W.},
  \bibinfo{year}{2003}.
\newblock \bibinfo{title}{${L}^p $ approximation of maps by diffeomorphisms}.
\newblock \bibinfo{journal}{Calculus of Variations and Partial Differential
  Equations} \bibinfo{volume}{16}, \bibinfo{pages}{147--164}.
\bibitem[{Chen and Tao(2021)}]{chen2021data}
\bibinfo{author}{Chen, R.}, \bibinfo{author}{Tao, M.}, \bibinfo{year}{2021}.
\newblock \bibinfo{title}{Data-driven prediction of general hamiltonian
  dynamics via learning exactly-symplectic maps}, in:
  \bibinfo{booktitle}{Proceedings of the 38th International Conference on
  Machine Learning, {ICML} 2021}, \bibinfo{publisher}{{PMLR}}. pp.
  \bibinfo{pages}{1717--1727}.
\bibitem[{Chen et~al.(2019)Chen, Behrmann, Duvenaud and
  Jacobsen}]{chen2019residual}
\bibinfo{author}{Chen, T.Q.}, \bibinfo{author}{Behrmann, J.},
  \bibinfo{author}{Duvenaud, D.}, \bibinfo{author}{Jacobsen, J.},
  \bibinfo{year}{2019}.
\newblock \bibinfo{title}{Residual flows for invertible generative modeling},
  in: \bibinfo{booktitle}{Advances in Neural Information Processing Systems},
  pp. \bibinfo{pages}{9913--9923}.
\bibitem[{Chen et~al.(2018)Chen, Rubanova, Bettencourt and
  Duvenaud}]{chen2018neural}
\bibinfo{author}{Chen, T.Q.}, \bibinfo{author}{Rubanova, Y.},
  \bibinfo{author}{Bettencourt, J.}, \bibinfo{author}{Duvenaud, D.},
  \bibinfo{year}{2018}.
\newblock \bibinfo{title}{Neural ordinary differential equations}, in:
  \bibinfo{booktitle}{Advances in Neural Information Processing Systems 31:
  Annual Conference on Neural Information Processing Systems 2018, NeurIPS
  2018, December 3-8, 2018, Montr{\'{e}}al, Canada}, pp.
  \bibinfo{pages}{6572--6583}.
\bibitem[{Cybenko(1989)}]{cybenko1989approximation}
\bibinfo{author}{Cybenko, G.}, \bibinfo{year}{1989}.
\newblock \bibinfo{title}{Approximation by superpositions of a sigmoidal
  function}.
\newblock \bibinfo{journal}{Mathematics of control, signals and systems}
  \bibinfo{volume}{2}, \bibinfo{pages}{303--314}.
\bibitem[{Dinh et~al.(2015)Dinh, Krueger and Bengio}]{dinh2015nice}
\bibinfo{author}{Dinh, L.}, \bibinfo{author}{Krueger, D.},
  \bibinfo{author}{Bengio, Y.}, \bibinfo{year}{2015}.
\newblock \bibinfo{title}{{NICE:} non-linear independent components
  estimation}, in: \bibinfo{booktitle}{3rd International Conference on Learning
  Representations, {ICLR} 2015, San Diego, CA, USA, May 7-9, 2015, Workshop
  Track Proceedings}.
\bibitem[{Dinh et~al.(2017)Dinh, Sohl{-}Dickstein and Bengio}]{dinh2017density}
\bibinfo{author}{Dinh, L.}, \bibinfo{author}{Sohl{-}Dickstein, J.},
  \bibinfo{author}{Bengio, S.}, \bibinfo{year}{2017}.
\newblock \bibinfo{title}{Density estimation using real {NVP}}, in:
  \bibinfo{booktitle}{5th International Conference on Learning Representations,
  {ICLR} 2017, Toulon, France, April 24-26, 2017, Conference Track
  Proceedings}, \bibinfo{publisher}{OpenReview.net}.
\bibitem[{Dupont et~al.(2019)Dupont, Doucet and Teh}]{dupont2019augmented}
\bibinfo{author}{Dupont, E.}, \bibinfo{author}{Doucet, A.},
  \bibinfo{author}{Teh, Y.W.}, \bibinfo{year}{2019}.
\newblock \bibinfo{title}{Augmented neural odes}, in:
  \bibinfo{booktitle}{Advances in Neural Information Processing Systems 32:
  Annual Conference on Neural Information Processing Systems 2019, NeurIPS
  2019, December 8-14, 2019, Vancouver, BC, Canada}, pp.
  \bibinfo{pages}{3134--3144}.
\bibitem[{E(2017)}]{e2017proposal}
\bibinfo{author}{E, W.}, \bibinfo{year}{2017}.
\newblock \bibinfo{title}{A proposal on machine learning via dynamical
  systems}.
\newblock \bibinfo{journal}{Communications in Mathematics and Statistics}
  \bibinfo{volume}{5}, \bibinfo{pages}{1--11}.
\bibitem[{E et~al.(2019)E, Han and Li}]{e2019mean}
\bibinfo{author}{E, W.}, \bibinfo{author}{Han, J.}, \bibinfo{author}{Li, Q.},
  \bibinfo{year}{2019}.
\newblock \bibinfo{title}{A mean-field optimal control formulation of deep
  learning}.
\newblock \bibinfo{journal}{Research in the Mathematical Sciences}
  \bibinfo{volume}{6}, \bibinfo{pages}{1--41}.
\bibitem[{Feng and Shang(1995)}]{feng1995volume}
\bibinfo{author}{Feng, K.}, \bibinfo{author}{Shang, Z.}, \bibinfo{year}{1995}.
\newblock \bibinfo{title}{Volume-preserving algorithms for source-free
  dynamical systems}.
\newblock \bibinfo{journal}{Numerische Mathematik} \bibinfo{volume}{71},
  \bibinfo{pages}{451--463}.
\bibitem[{Fiori(2011a)}]{fiori2011numerical}
\bibinfo{author}{Fiori, S.}, \bibinfo{year}{2011}a.
\newblock \bibinfo{title}{Extended {Hamiltonian} learning on {Riemannian}
  manifolds: {Numerical} aspects}.
\newblock \bibinfo{journal}{IEEE Transactions on Neural Networks and Learning
  Systems} \bibinfo{volume}{23}, \bibinfo{pages}{7--21}.
\bibitem[{Fiori(2011b)}]{fiori2011theoretical}
\bibinfo{author}{Fiori, S.}, \bibinfo{year}{2011}b.
\newblock \bibinfo{title}{Extended {Hamiltonian} learning on {Riemannian}
  manifolds: {Theoretical} aspects}.
\newblock \bibinfo{journal}{IEEE transactions on neural networks}
  \bibinfo{volume}{22}, \bibinfo{pages}{687--700}.
\bibitem[{Gomez et~al.(2017)Gomez, Ren, Urtasun and Grosse}]{gomez2017the}
\bibinfo{author}{Gomez, A.N.}, \bibinfo{author}{Ren, M.},
  \bibinfo{author}{Urtasun, R.}, \bibinfo{author}{Grosse, R.B.},
  \bibinfo{year}{2017}.
\newblock \bibinfo{title}{The reversible residual network: Backpropagation
  without storing activations}, in: \bibinfo{booktitle}{Advances in Neural
  Information Processing Systems 30: Annual Conference on Neural Information
  Processing Systems 2017, 4-9 December 2017, Long Beach, CA, {USA}}, pp.
  \bibinfo{pages}{2214--2224}.
\bibitem[{Greydanus et~al.(2019)Greydanus, Dzamba and
  Yosinski}]{greydanus2019hamiltonian}
\bibinfo{author}{Greydanus, S.}, \bibinfo{author}{Dzamba, M.},
  \bibinfo{author}{Yosinski, J.}, \bibinfo{year}{2019}.
\newblock \bibinfo{title}{Hamiltonian neural networks}, in:
  \bibinfo{booktitle}{Advances in Neural Information Processing Systems 32:
  Annual Conference on Neural Information Processing Systems 2019, NeurIPS
  2019, December 8-14, 2019, Vancouver, BC, Canada}, pp.
  \bibinfo{pages}{15353--15363}.
\bibitem[{Hairer et~al.(2006)Hairer, Lubich and Wanner}]{hairer2006geometric}
\bibinfo{author}{Hairer, E.}, \bibinfo{author}{Lubich, C.},
  \bibinfo{author}{Wanner, G.}, \bibinfo{year}{2006}.
\newblock \bibinfo{title}{Geometric numerical integration: structure-preserving
  algorithms for ordinary differential equations}. volume~\bibinfo{volume}{31}.
\newblock \bibinfo{publisher}{Springer Science \& Business Media}.
\bibitem[{Hairer et~al.(1993)Hairer, Norsett and Wanner}]{hairer1993solving}
\bibinfo{author}{Hairer, E.}, \bibinfo{author}{Norsett, S.},
  \bibinfo{author}{Wanner, G.}, \bibinfo{year}{1993}.
\newblock \bibinfo{title}{Solving Ordinary Differential Equations I: Nonstiff
  Problems}. volume~\bibinfo{volume}{8}.
\newblock \bibinfo{publisher}{Springer-Verlag, Berlin}.
\bibitem[{He et~al.(2016)He, Zhang, Ren and Sun}]{he2016deep}
\bibinfo{author}{He, K.}, \bibinfo{author}{Zhang, X.}, \bibinfo{author}{Ren,
  S.}, \bibinfo{author}{Sun, J.}, \bibinfo{year}{2016}.
\newblock \bibinfo{title}{Deep residual learning for image recognition}, in:
  \bibinfo{booktitle}{2016 {IEEE} Conference on Computer Vision and Pattern
  Recognition, {CVPR} 2016, Las Vegas, NV, USA, June 27-30, 2016},
  \bibinfo{publisher}{{IEEE} Computer Society}. pp. \bibinfo{pages}{770--778}.
\bibitem[{Hornik et~al.(1990)Hornik, Stinchcombe and
  White}]{hornik1990universal}
\bibinfo{author}{Hornik, K.}, \bibinfo{author}{Stinchcombe, M.},
  \bibinfo{author}{White, H.}, \bibinfo{year}{1990}.
\newblock \bibinfo{title}{Universal approximation of an unknown mapping and its
  derivatives using multilayer feedforward networks}.
\newblock \bibinfo{journal}{Neural Networks} \bibinfo{volume}{3},
  \bibinfo{pages}{551 -- 560}.
\bibitem[{Huang et~al.(2018)Huang, Krueger, Lacoste and
  Courville}]{huang2018neural}
\bibinfo{author}{Huang, C.}, \bibinfo{author}{Krueger, D.},
  \bibinfo{author}{Lacoste, A.}, \bibinfo{author}{Courville, A.C.},
  \bibinfo{year}{2018}.
\newblock \bibinfo{title}{Neural autoregressive flows}, in:
  \bibinfo{booktitle}{Proceedings of the 35th International Conference on
  Machine Learning, {ICML} 2018, Stockholmsm{\"{a}}ssan, Stockholm, Sweden,
  July 10-15, 2018}, \bibinfo{publisher}{{PMLR}}. pp.
  \bibinfo{pages}{2083--2092}.
\bibitem[{Jin et~al.(2020a)Jin, Lu, Tang and Karniadakis}]{jin2020quantifying}
\bibinfo{author}{Jin, P.}, \bibinfo{author}{Lu, L.}, \bibinfo{author}{Tang,
  Y.}, \bibinfo{author}{Karniadakis, G.E.}, \bibinfo{year}{2020}a.
\newblock \bibinfo{title}{Quantifying the generalization error in deep learning
  in terms of data distribution and neural network smoothness}.
\newblock \bibinfo{journal}{Neural Networks} \bibinfo{volume}{130},
  \bibinfo{pages}{85--99}.
\bibitem[{Jin et~al.(2020b)Jin, Zhang, Kevrekidis and
  Karniadakis}]{jin2020learning}
\bibinfo{author}{Jin, P.}, \bibinfo{author}{Zhang, Z.},
  \bibinfo{author}{Kevrekidis, I.G.}, \bibinfo{author}{Karniadakis, G.E.},
  \bibinfo{year}{2020}b.
\newblock \bibinfo{title}{Learning poisson systems and trajectories of
  autonomous systems via poisson neural networks}.
\newblock \bibinfo{journal}{arXiv preprint arXiv:2012.03133} .
\bibitem[{Jin et~al.(2020c)Jin, Zhang, Zhu, Tang and
  Karniadakis}]{jin2020sympnets}
\bibinfo{author}{Jin, P.}, \bibinfo{author}{Zhang, Z.}, \bibinfo{author}{Zhu,
  A.}, \bibinfo{author}{Tang, Y.}, \bibinfo{author}{Karniadakis, G.E.},
  \bibinfo{year}{2020}c.
\newblock \bibinfo{title}{Sympnets: Intrinsic structure-preserving symplectic
  networks for identifying hamiltonian systems}.
\newblock \bibinfo{journal}{Neural Networks} \bibinfo{volume}{132},
  \bibinfo{pages}{166 -- 179}.
\bibitem[{Kingma and Dhariwal(2018)}]{kingma2018glow}
\bibinfo{author}{Kingma, D.P.}, \bibinfo{author}{Dhariwal, P.},
  \bibinfo{year}{2018}.
\newblock \bibinfo{title}{Glow: Generative flow with invertible 1x1
  convolutions}, in: \bibinfo{booktitle}{Advances in Neural Information
  Processing Systems 31: Annual Conference on Neural Information Processing
  Systems 2018, NeurIPS 2018, December 3-8, 2018, Montr{\'{e}}al, Canada}, pp.
  \bibinfo{pages}{10236--10245}.
\bibitem[{Kong and Chaudhuri(2020)}]{kong2020the}
\bibinfo{author}{Kong, Z.}, \bibinfo{author}{Chaudhuri, K.},
  \bibinfo{year}{2020}.
\newblock \bibinfo{title}{The expressive power of a class of normalizing flow
  models}, in: \bibinfo{booktitle}{The 23rd International Conference on
  Artificial Intelligence and Statistics, {AISTATS} 2020, 26-28 August 2020,
  Online [Palermo, Sicily, Italy]}, \bibinfo{publisher}{{PMLR}}. pp.
  \bibinfo{pages}{3599--3609}.
\bibitem[{Krizhevsky et~al.(2012)Krizhevsky, Sutskever and
  Hinton}]{krizhevsky2012imagenet}
\bibinfo{author}{Krizhevsky, A.}, \bibinfo{author}{Sutskever, I.},
  \bibinfo{author}{Hinton, G.E.}, \bibinfo{year}{2012}.
\newblock \bibinfo{title}{Imagenet classification with deep convolutional
  neural networks}, in: \bibinfo{booktitle}{Advances in Neural Information
  Processing Systems 25, December 3-6, 2012, Lake Tahoe, Nevada, United
  States}, pp. \bibinfo{pages}{1106--1114}.
\bibitem[{Krizhevsky et~al.(2017)Krizhevsky, Sutskever and
  Hinton}]{krizhevsky2017imagenet}
\bibinfo{author}{Krizhevsky, A.}, \bibinfo{author}{Sutskever, I.},
  \bibinfo{author}{Hinton, G.E.}, \bibinfo{year}{2017}.
\newblock \bibinfo{title}{Imagenet classification with deep convolutional
  neural networks}.
\newblock \bibinfo{journal}{Commun. {ACM}} \bibinfo{volume}{60},
  \bibinfo{pages}{84--90}.
\bibitem[{LeCun et~al.(2015)LeCun, Bengio and Hinton}]{lecun2015deep}
\bibinfo{author}{LeCun, Y.}, \bibinfo{author}{Bengio, Y.},
  \bibinfo{author}{Hinton, G.E.}, \bibinfo{year}{2015}.
\newblock \bibinfo{title}{Deep learning}.
\newblock \bibinfo{journal}{Nature} \bibinfo{volume}{521},
  \bibinfo{pages}{436--444}.
\bibitem[{Li et~al.(2017)Li, Chen, Tai and Weinan}]{li2017maximum}
\bibinfo{author}{Li, Q.}, \bibinfo{author}{Chen, L.}, \bibinfo{author}{Tai,
  C.}, \bibinfo{author}{Weinan, E.}, \bibinfo{year}{2017}.
\newblock \bibinfo{title}{Maximum principle based algorithms for deep
  learning}.
\newblock \bibinfo{journal}{The Journal of Machine Learning Research}
  \bibinfo{volume}{18}, \bibinfo{pages}{5998--6026}.
\bibitem[{Li et~al.(2020)Li, Lin and Shen}]{li2020deep}
\bibinfo{author}{Li, Q.}, \bibinfo{author}{Lin, T.}, \bibinfo{author}{Shen,
  Z.}, \bibinfo{year}{2020}.
\newblock \bibinfo{title}{Deep learning via dynamical systems: An approximation
  perspective}.
\newblock \bibinfo{journal}{arXiv preprint arXiv:1912.10382} .
\bibitem[{Lu et~al.(2021a)Lu, Jin, Pang, Zhang and
  Karniadakis}]{lu2021learning}
\bibinfo{author}{Lu, L.}, \bibinfo{author}{Jin, P.}, \bibinfo{author}{Pang,
  G.}, \bibinfo{author}{Zhang, Z.}, \bibinfo{author}{Karniadakis, G.E.},
  \bibinfo{year}{2021}a.
\newblock \bibinfo{title}{Learning nonlinear operators via deeponet based on
  the universal approximation theorem of operators}.
\newblock \bibinfo{journal}{Nature Machine Intelligence} \bibinfo{volume}{3},
  \bibinfo{pages}{218--229}.
\bibitem[{Lu et~al.(2021b)Lu, Meng, Mao and Karniadakis}]{lu2021deepxde}
\bibinfo{author}{Lu, L.}, \bibinfo{author}{Meng, X.}, \bibinfo{author}{Mao,
  Z.}, \bibinfo{author}{Karniadakis, G.E.}, \bibinfo{year}{2021}b.
\newblock \bibinfo{title}{Deepxde: A deep learning library for solving
  differential equations}.
\newblock \bibinfo{journal}{SIAM Review} \bibinfo{volume}{63},
  \bibinfo{pages}{208--228}.
\bibitem[{Maas et~al.(2013)Maas, Hannun and Ng}]{maas2013rectifier}
\bibinfo{author}{Maas, A.L.}, \bibinfo{author}{Hannun, A.Y.},
  \bibinfo{author}{Ng, A.Y.}, \bibinfo{year}{2013}.
\newblock \bibinfo{title}{Rectifier nonlinearities improve neural network
  acoustic models}, in: \bibinfo{booktitle}{Proc. icml}, p.~\bibinfo{pages}{3}.
\bibitem[{Rezende and Mohamed(2015)}]{Rezende2015variational}
\bibinfo{author}{Rezende, D.J.}, \bibinfo{author}{Mohamed, S.},
  \bibinfo{year}{2015}.
\newblock \bibinfo{title}{Variational inference with normalizing flows}, in:
  \bibinfo{booktitle}{Proceedings of the 32nd International Conference on
  Machine Learning, {ICML} 2015, Lille, France, 6-11 July 2015},
  \bibinfo{publisher}{JMLR.org}. pp. \bibinfo{pages}{1530--1538}.
\bibitem[{Schmidhuber(2015)}]{schmidhuber2015deep}
\bibinfo{author}{Schmidhuber, J.}, \bibinfo{year}{2015}.
\newblock \bibinfo{title}{Deep learning in neural networks: An overview}.
\newblock \bibinfo{journal}{Neural Networks} \bibinfo{volume}{61},
  \bibinfo{pages}{85--117}.
\bibitem[{Shen et~al.(2021)Shen, Yang and Zhang}]{shen2021neural}
\bibinfo{author}{Shen, Z.}, \bibinfo{author}{Yang, H.}, \bibinfo{author}{Zhang,
  S.}, \bibinfo{year}{2021}.
\newblock \bibinfo{title}{Neural network approximation: Three hidden layers are
  enough}.
\newblock \bibinfo{journal}{Neural Networks} \bibinfo{volume}{141},
  \bibinfo{pages}{160--173}.
\bibitem[{Turaev(2002)}]{turaev2002polynomial}
\bibinfo{author}{Turaev, D.}, \bibinfo{year}{2002}.
\newblock \bibinfo{title}{Polynomial approximations of symplectic dynamics and
  richness of chaos in non-hyperbolic area-preserving maps}.
\newblock \bibinfo{journal}{Nonlinearity} \bibinfo{volume}{16},
  \bibinfo{pages}{123}.
\bibitem[{Zhang et~al.(2020)Zhang, Gao, Unterman and
  Arodz}]{zhang2020approximation}
\bibinfo{author}{Zhang, H.}, \bibinfo{author}{Gao, X.},
  \bibinfo{author}{Unterman, J.}, \bibinfo{author}{Arodz, T.},
  \bibinfo{year}{2020}.
\newblock \bibinfo{title}{Approximation capabilities of neural odes and
  invertible residual networks}, in: \bibinfo{booktitle}{Proceedings of the
  37th International Conference on Machine Learning, {ICML} 2020, 13-18 July
  2020, Virtual Event}, \bibinfo{publisher}{{PMLR}}. pp.
  \bibinfo{pages}{11086--11095}.
\bibitem[{Zhang et~al.(2021)Zhang, Zhang, Kang and Li}]{zhang2021ivpf}
\bibinfo{author}{Zhang, S.}, \bibinfo{author}{Zhang, C.},
  \bibinfo{author}{Kang, N.}, \bibinfo{author}{Li, Z.}, \bibinfo{year}{2021}.
\newblock \bibinfo{title}{ivpf: Numerical invertible volume preserving flow for
  efficient lossless compression}, in: \bibinfo{booktitle}{{IEEE} Conference on
  Computer Vision and Pattern Recognition, {CVPR} 2021, virtual, June 19-25,
  2021}, \bibinfo{publisher}{Computer Vision Foundation / {IEEE}}. pp.
  \bibinfo{pages}{620--629}.

\end{thebibliography}





\end{document}